\documentclass{article}

\PassOptionsToPackage{numbers, compress}{natbib}

\usepackage[final]{neurips_2019}

\usepackage[utf8]{inputenc} 
\usepackage[T1]{fontenc}    
\usepackage{hyperref}       
\usepackage{url}            
\usepackage{booktabs}       
\usepackage{amsfonts}       
\usepackage{nicefrac}       
\usepackage{microtype}      

\usepackage{lipsum}                     
\usepackage{xargs}                      
\usepackage[pdftex,dvipsnames]{xcolor}  
\usepackage[colorinlistoftodos,prependcaption,textsize=small]{todonotes}
\newcommandx{\unsure}[2][1=]{\todo[linecolor=red,backgroundcolor=red!25,bordercolor=red,#1]{#2}}
\newcommandx{\change}[2][1=]{\todo[linecolor=blue,backgroundcolor=blue!25,bordercolor=blue,#1]{#2}}
\newcommandx{\info}[2][1=]{\todo[linecolor=OliveGreen,backgroundcolor=OliveGreen!25,bordercolor=OliveGreen,#1]{#2}}
\newcommandx{\improvement}[2][1=]{\todo[linecolor=Plum,backgroundcolor=Plum!25,bordercolor=Plum,#1]{#2}}

\usepackage{caption}
\usepackage{enumitem}
\usepackage{wrapfig}
\usepackage[]{authblk}

\newcommand{\loss}{\mathcal{L}}
\newcommand{\thpub}{\theta_{\mathrm{pub}}}
\newcommand{\thpriv}{\theta_{\mathrm{priv}}}
\newcommand{\ypub}{y_{\mathrm{pub}}}
\newcommand{\ypriv}{y_{\mathrm{priv}}}
\newcommand{\lpub}{\loss_{\mathrm{pub}}}
\newcommand{\lpriv}{\loss_{\mathrm{priv}}}
\newcommand{\freeze}{\mathsf{Frozen}}

\title{Partially Encrypted Machine Learning \\ using Functional Encryption}

\author[1, 2]{\large Théo Ryffel}
\author[1]{Edouard Dufour-Sans}
\author[1,3]{Romain Gay}
\author[2, 1]{~\\Francis Bach}
\author[1, 2]{David Pointcheval}
\affil[1]{\normalsize Département d'informatique de l'ENS, ENS, CNRS, PSL University, Paris, France}
\affil[2]{INRIA, Paris, France}
\affil[3]{University of California, Berkeley}

\usepackage{amsmath}
\usepackage{amssymb}
\usepackage{mathtools}
\usepackage{resizegather}
\usepackage{thm-restate}

\usepackage{xspace}
\newtheorem{theorem}{Theorem}[section]
\newtheorem{lemma}[theorem]{Lemma}

\newtheorem{definition}[theorem]{Definition}

\newenvironment{proof}[1][ 
\noindent\textit{Proof}]{
    \vspace*{-\parskip}\noindent\textit{#1.}}{\hfill $\square$
    \medskip}

\newcommand{\cO}{\mathcal{O}}
\newcommand{\append}{\mathsf{append}}
\newcommand{\add}{\mathsf{add}}
\newcommand{\pair}{\mathsf{pair}}
\newcommand{\eq}{\mathsf{eq}}
\newcommand{\chal}{\mathsf{chal}}
\newcommand{\adv}{\mathsf{Adv}}

\newcommand{\Z}{\mathbb{Z}}
\newcommand{\N}{\mathbb{N}}
\newcommand{\Gone}{\mathbb{G}_1}
\newcommand{\Gtwo}{\mathbb{G}_2}
\newcommand{\Gt}{\mathbb{G}_T}
\newcommand{\matM}{\mathbf{M}}
\newcommand{\matW}{\mathbf{W}}

\newcommand{\GL}{\mathsf{GL}}

\DeclareMathOperator{\negl}{negl}
\newcommand{\zero}{\mathbf{0}}

\newcommand{\vs}{\vec{s}}
\newcommand{\vt}{\vec{t}}

\newcommand{\vx}{\vec{x}}
\newcommand{\vy}{\vec{y}}

\makeatletter
\newcommand\suchthat{%
 \@ifstar
  {\mathrel{}\middle|\mathrel{}}
  {\mid}%
}
\makeatother

\newcommand\numberthis{\addtocounter{equation}{1}\tag{\theequation}}

\newcommand{\getsr}{\stackrel{{}_\$}{\leftarrow}}

\newcommand{\secpar}{\lambda}

\newcommand{\advA}{\mathcal{A}}

\newcommand{\Exp}{\mathsf{Exp}}

\newcommand{\sk}{\mathsf{sk}}
\newcommand{\pk}{\mathsf{pk}}

\newcommand{\dk}{\mathsf{dk}}
\newcommand{\msk}{\mathsf{msk}}

\newcommand{\Enc}[1][]{{\mathsf{Enc}^{#1}}}
\newcommand{\Dec}[1][]{{\mathsf{Dec}^{#1}}}

\newcommand{\Adv}{\mathsf{Adv}}

\newcommand{\SetUp}{\mathsf{SetUp}\xspace}

\newcommand{\KeyGen}{\mathsf{KeyGen}\xspace}

\newcommand{\ggen}{\mathsf{GGen}}
\newcommand{\PG}{\mathcal{PG}}
\newcommand{\FE}{\mathsf{FE}}

\newcommand{\cY}{\mathcal{Y}}
\newcommand{\ct}{\mathsf{ct}}
\newcommand{\zo}{\{0,1\}}
\newcommand{\cF}{\mathcal{F}}
\newcommand{\cX}{\mathcal{X}}

\newcommand{\R}{\mathbb{R}}
\newcommand{\mat}[1]{\mathbf{#1}}

\newcommand{\argmax}{\operatornamewithlimits{argmax}}

\begin{document}

\maketitle
\vspace{-34px}
\begin{center}
\texttt{\{theo.ryffel,edufoursans,romain.gay,francis.bach,david.pointcheval\}@ens.fr}
\end{center}
\vspace{30px}

\begin{abstract}
  Machine learning on encrypted data has received a lot of attention thanks to recent breakthroughs in homomorphic encryption and secure multi-party computation. It allows outsourcing computation to untrusted servers without sacrificing privacy of sensitive data.
  We propose a practical framework to perform partially encrypted and privacy-preserving predictions which combines adversarial training and functional encryption.
  We first present a new functional encryption scheme to efficiently compute quadratic functions so that the data owner controls what can be computed but is not involved in the calculation: it provides a decryption key which allows one to learn a specific function evaluation of some encrypted data.
  We then show how to use it in machine learning to partially encrypt neural networks with quadratic activation functions at evaluation time, and we provide a thorough analysis of the information leaks based on indistinguishability of data items of the same label.
  Last, since most encryption schemes cannot deal with the last thresholding operation used for classification, we propose a training method to prevent selected sensitive features from leaking, which adversarially optimizes the network against an adversary trying to identify these features. This is interesting for several existing works using partially encrypted machine learning as it comes with little reduction on the model's accuracy and significantly improves data privacy.

\end{abstract}

\section{Introduction} \vspace{-1px}

As both public opinion and regulators are becoming increasingly aware of issues of data privacy, the area of privacy-preserving machine learning has emerged with the aim of reshaping the way machine learning deals with private data. Breakthroughs in fully homomorphic encryption (FHE) \citep{FastHEofDiNN,FastFHE01sec} and secure multi-party computation (SMPC) \citep{MPCfromSHE,Securenn} have made computation on encrypted data practical and implementations of neural networks to do encrypted predictions have flourished \citep{Chameleon,PySyft,sealcrypto,barni2011privacy,bost2015machine}.

However, these protocols require the data owner encrypting the inputs and the parties performing the computations to interact and communicate in order to get decrypted results, which we would like to avoid in some cases, like spam filtering, for example, where the email receiver should not need to be online for the email server to classify incoming email as spam or not. Functional encryption (FE) \citep{TCC:BonSahWat11,cryptoeprint:2010:556} in return does not need interaction to compute over encrypted data: it allows users to receive in plaintext specific functional evaluations of encrypted data: for a function $f$, a functional decryption key can be generated such that, given any ciphertext with underlying plaintext $x$, a user can use this key to obtain $f(x)$ without learning $x$ or any other information than $f(x)$. It stands in between traditional public key encryption, where data can only be directly revealed, and FHE, where data can be manipulated but cannot be revealed: it allows the user to tightly control what is disclosed about his data.

\subsection{Use cases}

\textbf{Spam filtering.} Consider the following scenario: Alice uses a secure email protocol which makes use of functional encryption. Bob uses Alice's public key to send her an email, which lands on Alice's email provider's server. Alice gave the server keys that enable it to process the email and take a predefined set of appropriate actions without her being online. The server could learn how urgent the email is and decide accordingly whether to alert Alice. It could also detect whether the message is spam and store it in the spam box right away.

\textbf{Privacy-preserving enforcement of content policies} Another use case could be to enable platforms, such as messaging apps, to maintain user privacy through end-to-end encryption, while filtering out content that is illegal or doesn't adhere to policies the site may have regarding, for instance, abusive speech or explicit images.


These applications are not currently feasible within a reasonable computing time, as the construction of FE for all kinds of circuits is essentially equivalent to \emph{indistinguishable obfuscation} \cite{C:BGIRSVY01,FOCS:GGHRSW13}, concrete instances of which have been shown insecure, let alone efficient.  However, there exist practical FE schemes for the inner-product functionality \citep{PKC:ABDP15,C:AgrLibSte16} and more recently for quadratic computations \citep{C:BCFG17}, that is usable for practical applications.

\subsection{Our contributions}

We introduce a new FE scheme to compute quadratic forms which outperforms that of Baltico et al.~\cite{C:BCFG17} in terms of complexity, and provide an efficient implementation of this scheme. We show how to use it to build privacy preserving neural networks, which perform well on simple image classification problems. Specifically, we show that the first layers of a polynomial network can be run on encrypted inputs using this quadratic scheme.

In addition, we present an adversarial training technique to process these first layers to improve privacy, so that their output, which is in plaintext, cannot be used by adversaries to recover specific sensitive information at test time. This adversarial procedure is generic for semi-encrypted neural networks and aims at reducing the information leakage, as the decrypted output is not directly the classification result but an intermediate layer (i.e., the neuron outputs of the neural network before thresholding). This has been overlooked in other popular encrypted classification schemes (even in FHE-based constructions like \cite{cryptonets} and \cite{FastHEofDiNN}), where the $\mathsf{argmax}$ operation used to select the class label is made in clear, as it is either not possible with FE, or quite inefficient with FHE and SMPC.




We demonstrate the practicality of our approach using a dataset inspired from MNIST \citep{mnist}, which is made of images of digits written using two different fonts. We show how to perform classification of the encrypted digit images in less than $3$ seconds with over $97.7\%$ accuracy while making the font prediction a hard task for a whole set of adversaries.

This paper builds on a preliminary version available on the \hyperlink{https://eprint.iacr.org/2018/206}{Cryptology ePrint Archive} at \url{eprint.iacr.org/2018/206}. All code and implementations can be found online at \url{github.com/LaRiffle/collateral-learning} and \url{github.com/edufoursans/reading-in-the-dark}.

\section{Background Knowledge} \vspace{-1px}

\subsection{Quadratic and Polynomial Neural Networks}

Polynomial neural networks are a class of networks which only use linear elements, like fully connected linear layers, convolutions but with average pooling, and model activation functions with polynomial approximations when not simply the square function. Despite these simplifications, they have proved themselves satisfactorily accurate for relatively simple tasks (\cite{cryptonets} learns on MNIST and \cite{HCNNAlexNetMoment} on CIFAR10 \cite{cifar10}). The simplicity of the operations they build on guarantees good efficiency, especially for the gradient computations, and works like \cite{DBLP:journals/corr/LivniSS14} have shown that they can achieve convergence rates similar to those of networks with non-linear activations. 

In particular, they have been used for several early stage implementations in cryptography  \citep{cryptonets,FastFHE01sec,eprint:2017:1114} to demonstrate the usability of new protocols for machine learning. 
However, the $\mathsf{argmax}$ or other thresholding function present at the end of a classifier network to select the class among the output neurons cannot be conveniently handled, so several protocol implementations (among which ours) run polynomial networks on encrypted inputs, but take the $\mathsf{argmax}$ over the decrypted output of the network. This results in potential information leakage which could be maliciously exploited.

\subsection{Functional Encryption}\vspace{-3px} 

Functional encryption extends the notion of public key encryption where one uses a public key $\pk$ and a secret key $\sk$ to respectively encrypt and decrypt some data. More precisely, $\pk$ is still used to encrypt data, but for a given function $f$, $\sk$ can be used to derive a functional decryption key $\dk_{f}$ which will be shared to users so that, given a ciphertext of $x$, they can decrypt $f(x)$ but not $x$. In particular, someone having access to $\dk_{f}$ cannot learn anything about $x$ other than $f(x)$. Note also that functions cannot be composed, since the decryption happens within the function evaluation. Hence, only single quadratic functions can be securely evaluated. A formal definition of functional encryption is provided in Appendix \ref{appendix:FE}.

\textbf{Perfect correctness.} Perfect correctness is achieved in functional encryption: $\forall x \in \cX$, $f \in \cF$, $\Pr[\Dec(\dk_f,\ct)=f(x)]=1$, where  $\dk_f \gets \KeyGen(\msk,f)$ and $\ct \gets \Enc(\pk,x)$. Note that this property is a very strict condition, which is not satisfied by exisiting fully homomorphic encryption schemes (FHE), such as \citep{LWE,C:GenSahWat13}.

\subsection{Indistinguishability and security}\vspace{-3px} 

To assess the security of our framework, we first consider the FE scheme security and make sure that we cannot learn anything more than what the function is supposed to output given an encryption of $x$. Second, we analyze how sensitive the output $f(x)$ is with respect to the private input $x$. For both studies, we will rely on \textit{indistinguishability} \cite{GolMic84}, a classical security notion which can be summed up in the following game: an adversary provides two input items to the challenger (here our FE algorithm), and the challenger chooses one item to be encrypted, runs encryption on it before returning the output. The adversary should not be able to detect which input was used. This is known as IND-CPA security in cryptography and a formal definition of it can be found in Appendix \ref{app:ind-cpa}.

We will first prove that our quadratic FE scheme achieves IND-CPA security, then, we will use a relaxed version of indistinguishability to measure the FE output sensitivity. More precisely, we will make the hypothesis that our input data can be used to predict public labels but also sensitive private ones, respectively $\ypub$ and $\ypriv$. Our quadratic FE scheme $q$ aims at predicting $\ypub$ and an adversary would rather like to infer $\ypriv$. In this case, the security game consists in the adversary providing two inputs $(x_0, x_1)$ labelled with the same $\ypub$ but a different $\ypriv$ and then trying to distinguish which one was selected by the challenger, given its output $q(x_b)$, $b \in \{0, 1\}$. One way to do this is to measure the ability of an adversary to predict $\ypriv$ for items which all belong to the same $\ypub$ class. 
 
In particular, note that we do not consider approaches based on input reconstruction (as done by \cite{Carpov2018IlluminatingTD}) because in many cases, the adversary is not interested in reconstructing the whole input, but rather wants to get insights into specific characteristics.


Another way to see this problem is that we want the sensitive label $\ypriv$ to be independent from the decrypted output $q(x)$ (which is a proxy to the prediction), given the true public label $\ypub$. This independence notion is known as \textit{separation} and is used as a fairness criterion in \cite{fairnessML} if the sensitive features can be misused for discrimination.

 \vspace{-2px}
\section{Our Context for Private Inference} \vspace{-2px}
\subsection{Classifying in two directions} \vspace{-2px}

We are interested in specific types of datasets $(\vec x_i)_{i=1,...,n}$ which have public labels $\ypub$ but also private ones $\ypriv$. Moreover, these different types of labels should be entangled, meaning that they should not be easily separable, unlike the color and the shape of an object in an image for example which can be simply separated.
 For example, in the spam filtering use case mentioned above, $\ypub$ would be a spam flag, and $\ypriv$ would be some marketing information highlighting areas of interest of the email recipient like technology, culture, etc. In addition, to simplify our analysis, we assume that classes are balanced for all types of labels, and that labels are independent from each other given the input: $\forall \vec x,  P(\ypub, \ypriv | \vec x) = P(\ypub | \vec x) P(\ypriv | \vec x)$. 
To illustrate our approach in the case of image recognition, we propose a synthetic dataset inspired by MNIST which consists of 60 000 grey scaled images of $28\times 28$ pixels representing digits using two fonts and some distortion, as shown in Figure~\ref{fig:dataset}. Here, the public label $\ypub$ is the digit on the image and the private one $\ypriv$ is the font used to draw it.

\begin{figure}[htb!]
  \centering
  \includegraphics[width=0.9\linewidth]{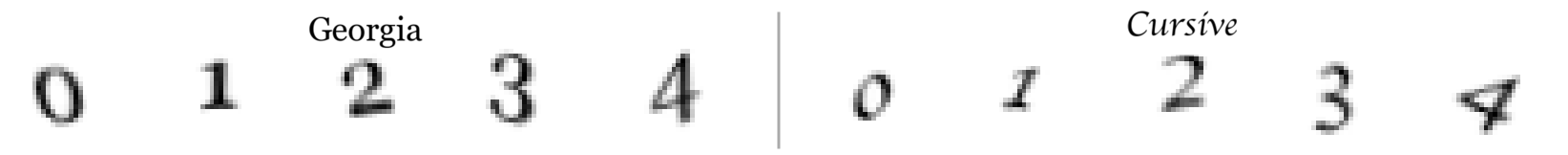}
  \caption{Artificial dataset inspired from MNIST with two types of labels.}\label{fig:dataset}
\end{figure}

We define two tasks: a \textit{main} task which tries to predict $\ypub$ using a partially-encrypted polynomial neural network with functional encryption, and a \textit{collateral} task which is performed by an adversary who tries to leverage the output of the FE encrypted network at test time to predict $\ypriv$. Our goal is to perform the main task with high accuracy while making the collateral one as bad as random predictions. In terms of indistinguishability, given a dataset with the same digit drawn, it should be infeasible to detect the used font.

\subsection{Equivalence with a Quadratic Functional Encryption scheme}

We now introduce our new framework for quadratic functional encryption and show that it can be used to partially encrypt a polynomial network.

\subsubsection{Functional Encryption for Quadratic Polynomials}

We build an efficient FE scheme for the set of quadratic functions defined as $\cF_{n,B_x,B_y,B_q} \subset \{ q: [-B_x,B_x]^n \times [-B_y,B_y]^n \rightarrow \Z\}$, where $q$ is described as a set of bounded coefficients $\{q_{i,j} \in [-B_q,B_q]\}_{i,j \in [n]}$ and for all vectors  $(\vx, \vec y)$, we have $\textstyle q(\vx,\vy) = \sum_{i,j \in [n]} q_{i,j} x_i y_j.$

A complete version of our scheme is given in Figure \ref{fig:FE}, but here are the main ideas and notations. First note that we use bilinear groups, \textit{i.e.}, a set of prime-order groups $\Gone$, $\Gtwo$ and $\Gt$ together with a bilinear map $e: \Gone \times \Gtwo \rightarrow \Gt$ called \textit{pairing} which satisfies $e(g_1^a , g_2^b) = e(g_1, g_2)^{ab}$ for any exponents $a,b \in \Z$: one can compute quadratic polynomials in the exponent. Here, $g_1$, $g_2$ are generators of $\Gone$ and $\Gtwo$ and $g_T := e(g_1, g_2)$ is a generator of the target group~$\Gt$. 
A pair of vectors $(\vs,\vt)$ is first selected and constitutes the private key $\msk$, while the public key is $(g_1^{\vs},g_2^{\vt})$.

Encrypting $(\vx,\vy)$ roughly consists of masking $g_1^{\vx}$ and $g_2^{\vy}$ with $g_1^{\vs}$ and $g_2^{\vt}$, which allows any user to compute $g_T^{q(\vx,\vy)-q(\vs,\vt)}$ with for any quadratic function $q$, using the pairing. The functional decryption key for a specific $q$ is $g_T^{q(\vs,\vt)}$ which allows to get $g_T^{q(\vx,\vy)}$. Last, taking the discrete logarithm gives access to $q(\vx,\vy)$ (discrete logarithm for small exponents is easy). Security uses the fact that it is hard to compute $\msk$ from $\pk$ (discrete logarithm for large exponents $\vs,\vt$ is hard to compute). More details are given in Appendix \ref{appendix:QFEproofs}\footnote{Note that we only present a simplified scheme here. In particular, the actual encryption is randomized, which is necessary to achieve IND-CPA security.}

\begin{figure*}[ht]
  \begin{center}
    \begin{tabular}{|l|}\hline
        \underline{$\SetUp(1^\lambda,\cF_{n,B_x,B_y,B_f})$:}\\
        $\PG:=(\Gone,\Gtwo,p,g_1,g_2,e)\gets \ggen(1^\lambda)$, $\vs, \vt \getsr \Z^n_p$,
        $\msk := (\vs,\vt)$, $\pk := \left(\PG,g_1^{\vs}, g_2^{\vt}\right)$\\
        Return $(\pk,\msk)$.\\
        \\
        \underline{$\Enc\big(\pk, (\vx,\vy)\big)$:}\\
        $\gamma \getsr \Z_p$, $\matW \getsr \GL_2$, for all $i \in [n]$, $\vec{a}_i := (\matW^{-1})^\top \begin{pmatrix}x_i \\ \gamma s_i\end{pmatrix}$, $\vec{b}_i := \matW \begin{pmatrix} y_i \\ -t_i\end{pmatrix}$\\
        Return $\ct := \left(g_1^{\gamma}, \{g_1^{\vec{a}_i},g_2^{\vec{b}_i}\}_{i \in [n]}\right) \in \Gone \times (\Gone^2 \times \Gtwo^2)^n$\\
        \\
        \underline{$\KeyGen(\msk,q)$:}\\
        Return $\dk_f := \left(g_2^{q(\vs,\vt)},q\right) \in \Gtwo \times \cF_{n,B_x,B_y,B_q}$.\\
        \\
        \underline{$\Dec\left(\pk,\ct := \left(g_1^\gamma, \{g_1^{\vec{a}_i},g_2^{\vec{b}_i}\}_{i \in [n]}\right) ,\dk_q := \left(g_2^{q(\vs,\vt)},q\right)\right)$:}\\[2ex]
        $out := e(g_1^\gamma,g_2^{q(\vs,\vt)})\cdot \prod_{i,j \in [n]} e\big(g_1^{\vec{a}_i},g_2^{\vec{b}_j}\big)^{q_{i,j}}$\\
        Return $\log(out) \in \Z$.
        \\[1.5ex]\hline
    \end{tabular}
    \caption{Our functional encryption scheme for quadratic polynomials.}\label{fig:FE}
  \end{center}
\end{figure*}

\begin{theorem}[Security, correctness and complexity]
    The FE scheme provided in Figure \ref{fig:FE}:
    \begin{itemize}[noitemsep,topsep=0pt,parsep=0pt,partopsep=0pt]
        \item is IND-CPA secure in the Generic Bilinear Group Model,
        \item verifies $\log(out) = q(\vx,\vy)$ and satisfies perfect correctness,
        \item has a overall decryption complexity of ~$ 2n^2(E+P) + P + D$,
    \end{itemize}
    where $E$, $P$ and $D$ respectively denote exponentiation, pairing and discrete logarithm complexities.
\end{theorem}

Our scheme outperforms previous schemes for quadratic FE with the same security assumption, like the one from \cite[Sec. 4]{C:BCFG17} which achieves $3n^2(E+P) + 2P + D$ complexity and uses larger ciphertexts and decryption keys. Note that the efficiency of the decryption can even be further optimized for those quadratic polynomials used that are relevant to our application (see Section \ref{sec:equivFENN}).

\textbf{Computing the discrete logarithm for decryption.}
Our decryption requires computing discrete logarithms of group elements in base $g_T$, but contrary to previous works like \cite{SCN:KLMMRW18} it is independent of the ciphertext and the functional decryption key used to decrypt. This allows to pre-compute values and dramatically speeds-up decryption.

\subsubsection{Equivalence of the FE scheme with a Quadratic Network}\label{sec:equivFENN}

We classify data which can be represented as a vector $\vec x \in [0,B]^n$ (in our case, the size $B=255$, and the dimension $n=784$) and we first build models $(q_i)_{i\in[\ell]}$ for each public label $i \in [\ell]$, such that our prediction $\ypub$ for $\vec x$ is $\argmax_{i\in[\ell]} q_i(\vec x)$.

\textbf{Quadratic polynomial on $\R^n$.} The most straightforward way to use our FE scheme would be for us to learn a model $(\mat Q_i)_{i\in[\ell]}\in {\left(\R^{n \times n}\right)}^{\ell}$, which we would then round onto   integers, such that $q_i(\vec x) = \vec x^\top \mat Q_i \vec x$, $\forall i \in [\ell]$. This is a unnecessarily powerful model in the case of MNIST as it has ${\ell}n^2$ parameters ($n=784$), and the resulting number of pairings to compute would be unreasonably large.

\textbf{Linear homomorphism.} The encryption algorithm of our FE scheme is linearly homomorphic with respect to the plaintext: given an encryption of $(\vec x, \vec y)$ under the secret key $\msk := (\vec s, \vec t)$, one can efficiently compute an encryption of $(\vec u^\top \vec x, \vec v^\top \vec y)$ under the secret key $\msk' := (\vec u^\top \vec s, \vec v^\top \vec t)$ for any linear combination $\vec u,\vec v$ (see proof in Appendix \ref{app:QFEequiv}). Any vector $\vec v$ is a column, and $\vec v^\top$ is a row.

Therefore, if $q$ can be written $q(\vec x, \vec y) = (\mat U \vec x)^\top \mat M (\mat V \vec y)$ for all ($\vec x$, $\vec y$), with $\mat U, \mat V \in \Z_p^{d \times n}$ projection matrices and $\mat M \in \Z_p^{d \times d}$, it is more efficient to first compute the encryption of $(\mat U \vec x, \mat V \vec y)$ from the encryption of $(\vec x, \vec y)$, and then to apply the functional decryption on these ciphertexts, because their underlying plaintexts are of reduced dimension $d < n$. This reduces the number of exponentiations from $2n^2$ to $2dn$ and the number of pairing computations from $2n^2$ to $2d^2$ for a single $q_i$. This is a major efficiency improvement for small $d$, as pairings are the main bottleneck in the computation.

\textbf{Projection and quadratic polynomial on $\R^d$.} We can use this and apply the quadratic polynomials on projected vectors: we learn $\mat P \in \R^{n\times{}d}$ and $(\mat Q_i)_{i\in[\ell]}\in {\left(\R^{d \times d}\right)}^{\ell}$, and our model is $q_i(\vec x) = (\mat P\vec x)^\top \mat Q_i (\mat P\vec x)$, $\forall i \in [\ell]$. We only need $2\ell d^2$ pairings and since the same $\mat P$ is used for all $q_i$, we only compute once the encryption of $\mat P\vec x$ from the encryption of $\vec x$. Better yet, we can also perform the pairings only once, and then compute the scores by exponentiating with different coefficients the same results of the pairings, thus only requiring $2d^2$ pairing evaluations, independently of $\ell$.


\textbf{Degree 2 polynomial network, with one hidden layer.} To further reduce the number of pairings, we actually limit ourselves to diagonal matrices, and thus rename $\mat Q_i$ to $\mat D_i$. We find that the gain in efficiency associated with only computing $2d$ pairings is worth the small drop in accuracy. The resulting model is actually a polynomial network of degree 2 with one hidden layer of $d$ neurons and the activation function is the square. In the following experiments we take $d = 40$.

Our final encrypted model can thus be written as
$q_i(\vec x) = (\mat P\vec x)^\top \mat D_i (\mat P\vec x), \forall i \in [\ell]$, where we add a bias term to $\vec x$ by replacing it with $\vec x = ( 1 ~ x_{1}  \dots   x_{n})$.

\textbf{Full network.} The result of the quadratic $(q_i(\vec x))_{i \in [\ell]}$ (i.e., of the private quadratic network) is now visible in clear. As mentioned above, we cannot compose this block several times as it contains decryption, so this is currently the best that we can have as an encrypted computation with FE.  Instead of simply applying the $\mathsf{argmax}$ to the cleartext output of this privately-evaluated quadratic network to get the label, we observe that adding more plaintext layers on top of it helps improving the overall accuracy of the main task. We have therefore a neural network composed of a private and a public part, as illustrated in Figure \ref{fig:encryptnet}.

\begin{figure}[!htb]
   \begin{minipage}{0.33\textwidth}
     \centering
     \includegraphics[width=1\linewidth]{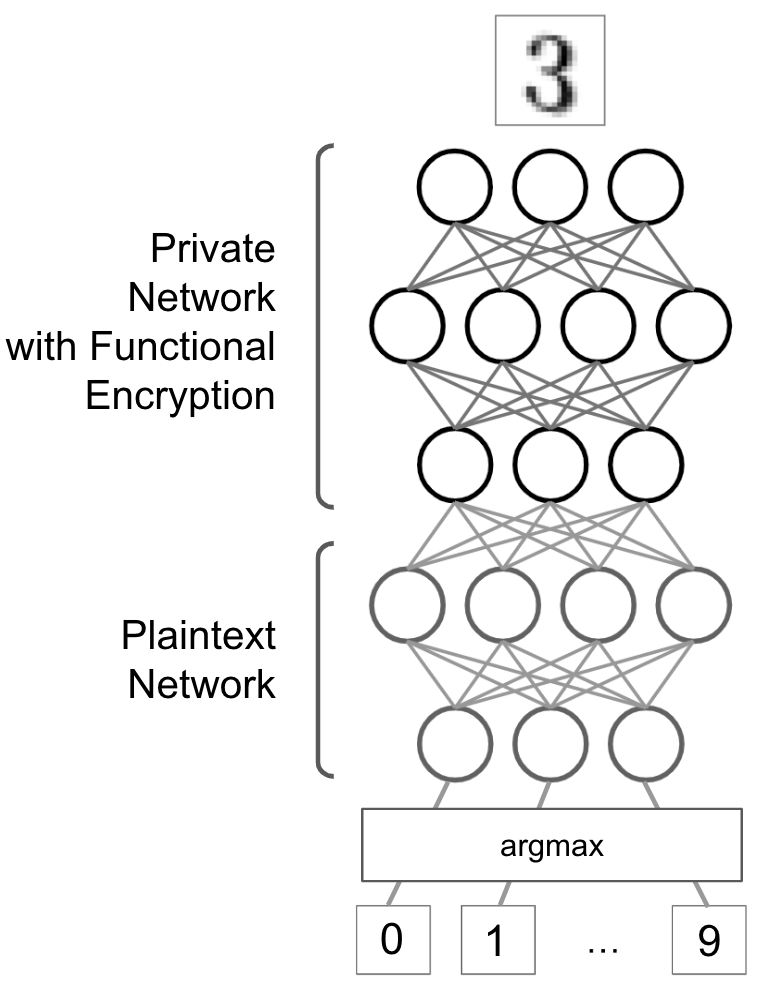}
      \caption{Semi-encrypted network using quadratic FE.}\label{fig:encryptnet}
   \end{minipage}\hfill
   \begin{minipage}{0.63\textwidth}
     \centering
     \includegraphics[width=1\linewidth]{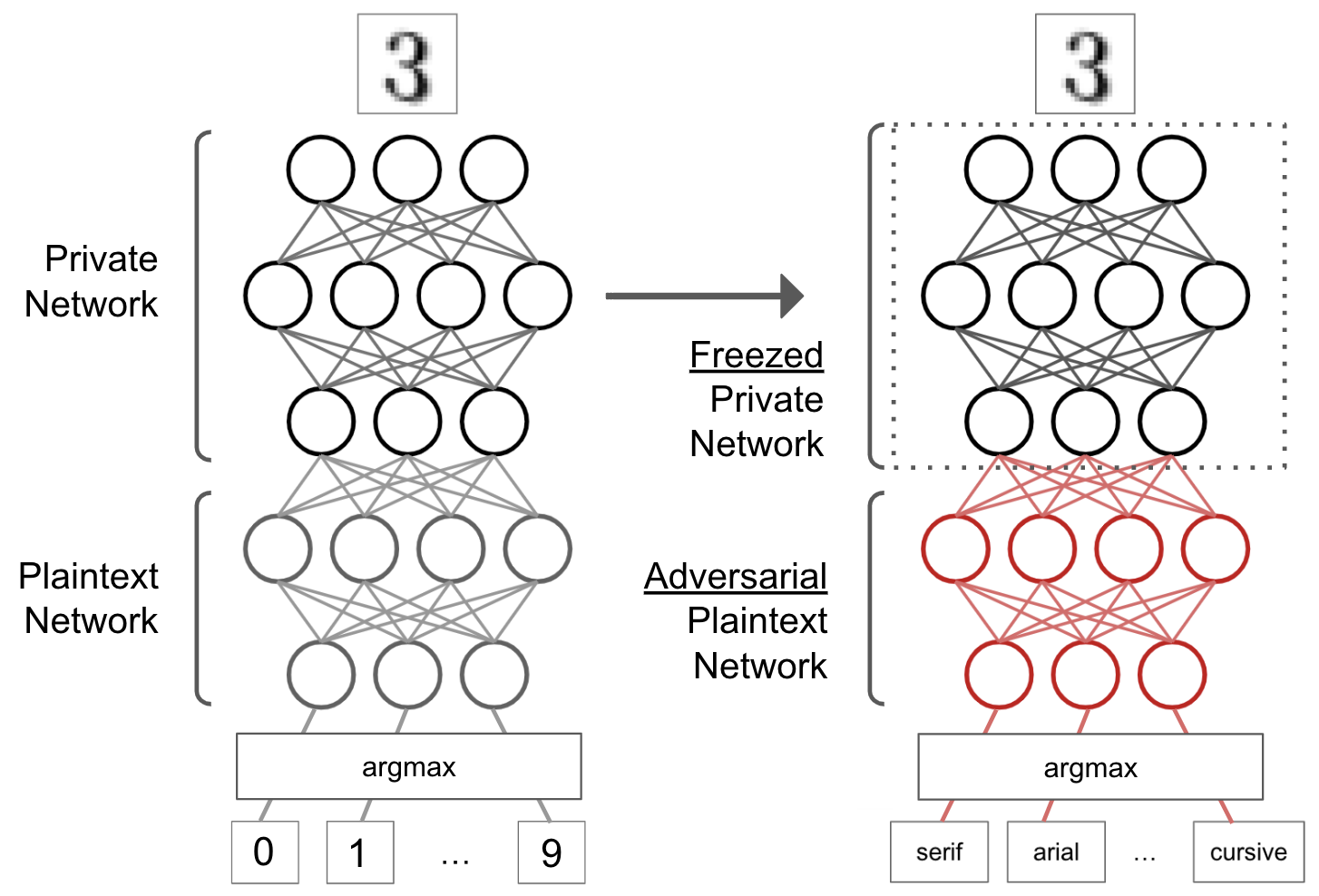}
     \caption{Semi-encrypted network with an adversary trying\\ to recover private labels from the private quadratic network.}\label{fig:encryptnetadv}
   \end{minipage}
\end{figure}\vspace{-10px}

\subsection{Threat of Collateral Learning}\label{sec:CLthreat}

A typical adversary would have a read access to the main task classification process. It would leverage the output of the quadratic network to try to learn the font used on ciphered images. To do this, all that is needed is to train another network on top of the quadratic network so that it learns to predict the font, assuming some access to labeled samples (which is the case if the adversary encrypts itself images and provides them to the main task at evaluation time). Note that in this case the private network is not updated by the collateral network as we assume it is only provided in read access after the main task is trained. Figure \ref{fig:encryptnetadv} summarizes the setting.

We implemented this scenario using as adversary a neural network composed of a first layer acting as a decompression step where we increase the number of neurons from $10$ back to $28 \times 28$ and add on top of it a classical\footnote{https://github.com/pytorch/examples/blob/master/mnist/main.py} convolutional neural network (CNN). This structure is reminiscent of autoencoders \cite{autoencoders} where the bottleneck is the public output of the private net and the challenge of this autoencoder is to correctly memorize the public label while forgetting the private one. What we observed is  striking: in less than $10$ epochs, the collateral network leverages the $10$ public neurons output and achieves $93.5\%$ accuracy for the font prediction. As expected, it gets even worse when the adversary is assessed with the indistinguishability criterion because in that case the adversary can work on a dataset where only a  specific digit is represented: this reduces the variability of the samples and makes it easier to distinguish the font; the probability of success is indeed of $96.9\%$.

We call \textit{collateral learning} this phenomenon of learning unexpected features and will show in the next section how to implement counter-measures to this threat in order to improve privacy.

\section{Defeating Collateral Learning} \vspace{-2px}
\subsection{Reducing information leakage}\label{ReduceLeak}

Our first approach is based on the observation that we leak many bits of information. We first investigate whether we can reduce the number of outputs of the privately-evaluted network, as adding extra layers on top of the private network makes it no longer necessary to keep $10$ of them. 
\begin{wrapfigure}{r}{0.5\textwidth}
  \centering
  \includegraphics[width=1\linewidth]{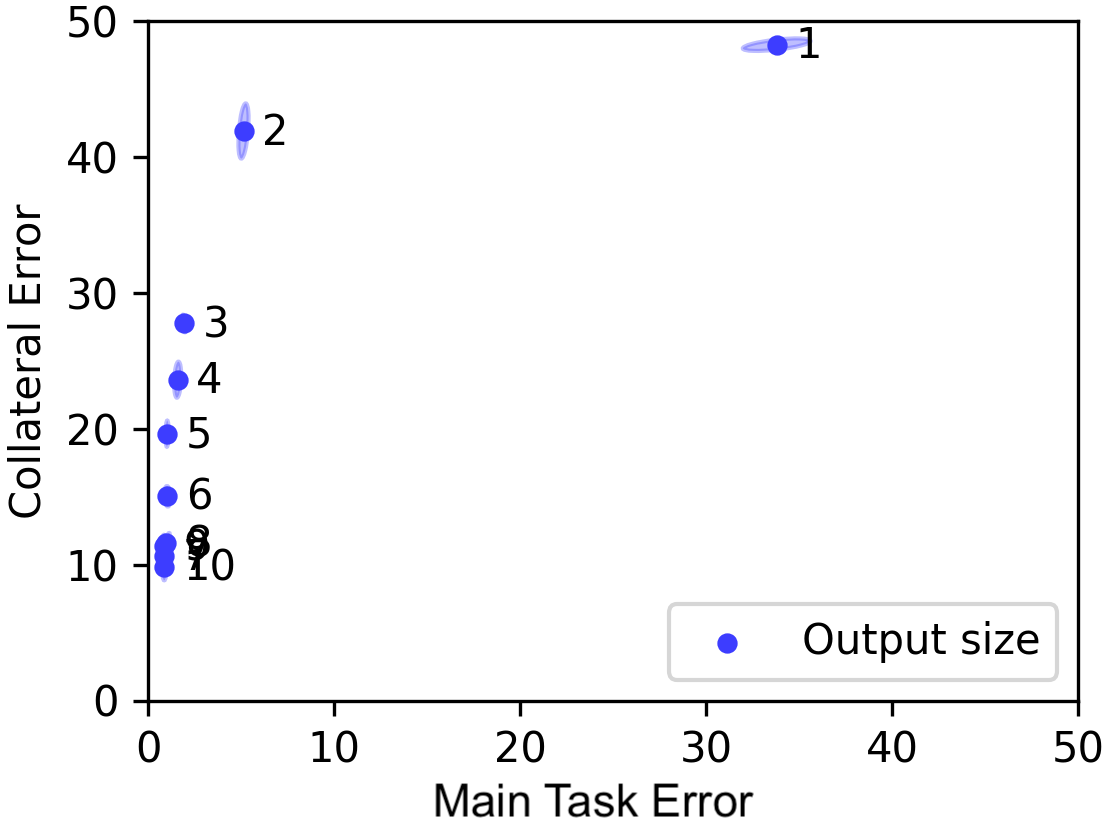}
  \caption{Trade-off between main and collateral accuracies depending on the private output size. }\label{fig:pareto_nosabo}
\end{wrapfigure}
The intuition is that if the information that is relevant to the main task can fit in less than 10 neurons, then the extra neurons would leak unnecessary information.
We have therefore a trade-off between reducing too much and losing desired information or keeping a too large output and having an important leakage. We can observe this through the respective accuracies as it is shown in Figure \ref{fig:pareto_nosabo}, where the main and adversarial networks are CNNs as in Section \ref{sec:CLthreat} with 10 epochs of training using 7-fold cross validation.
What we observe here is interesting: the main task does not exhibit significant weaknesses even with size 3 where we drop to $97.1\%$ which is still very good although $2\%$ under the best accuracy. In return, the collateral accuracy starts to significantly decrease when output size is below 7. At size 4, it is only $76.4\%$ on average so $18\%$ less than the baseline. We will keep an output size of 3 or 4 for the next experiments to keep the main accuracy almost unchanged.

Another hyperparameter that we can consider is the weight compression: how many bits do we need to represent the weights on the private networks layers? This is of interest for the FE scheme as we need to convert all weights to integers and those integers will be low provided that the compression rate is high. Small weight integers mean that the output of the private network has a relatively low amplitude and can be therefore efficiently decrypted using discrete logarithm. We managed to express all weights and even the input image using 4 bit values with limited impact on the main accuracy and almost none on the collateral one. Details about compression can be found in Appendix \ref{app:Wcompression}.

\subsection{Adversarial training}

We propose a new approach to actively adapt against collateral learning. The main idea is to simulate adversaries and to try to defeat them. To do this, we use semi-adversarial training and optimize simultaneously the main classification objective and the opposite of the collateral objective of a given simulated adversary. The function that we want to minimize at each iteration step can be written:
\[  \min_{\theta_q} [ 
    \min_{\thpub} \lpub(\theta_q, \thpub) - \alpha \min_{\thpriv} \lpriv(\theta_q, \thpriv) 
    ] .
\]
This approach is inspired from \cite{NIPS2017_6699} where the authors train some objective against nuisances parameters to build a classifier independent from these nuisances. Private features leaking in our scheme can indeed be considered to be a nuisance. However, our approach goes one step further as we do not just stack a network above another; our global network structure is fork-like: the common basis is the private network and the two forks are the main and collateral classifiers. This allows us to have a better classifier for the main task which is not as sensitive to the adversarial training as the scheme exposed by \cite[Figure 1]{NIPS2017_6699}. One other difference is that the collateral classifier is a specific modeling of an adversary, and we will discuss this in details in the next section.
We define in Figure \ref{fig:AdvTraining} the 3-step procedure used to implement this semi-adversarial training using partial back-propagation.

\begin{figure*}
  \begin{center}
    {\footnotesize
    \begin{tabular}{|l|}\hline
        \underline{Pre-training}: \textit{Initial phase where both tasks learn and strengthen before the joint optimization}\\
        Minimize $\lpub(\theta_q, \thpub)$\\
        Minimize $\lpriv(\freeze(\theta_q), \thpriv)$\\
        \\
        \underline{Semi-adversarial training}: \textit{The joint optimization phase, where $\thpub$ and $\thpriv$ are updated depending on}\\
        \textit{the variations of $\theta_q$ and $\theta_q$ is optimized to reduce the loss $\loss = \lpub - \alpha \lpriv$}\\
        Minimize $\lpub(\freeze(\theta_q), \thpub)$\\
        Minimize $\lpriv(\freeze(\theta_q), \thpriv)$\\
        Minimize $\loss = \lpub(\theta_q, \freeze(\thpub)) - \alpha \lpriv(\theta_q, \freeze(\thpriv)) $\\
        \\
        \underline{Recover phase}: \textit{Both tasks recover from the perturbations induced by the adversarial phase, $\theta_q$ does not}\\
        \textit{change anymore}\\
        Minimize $\lpub(\freeze(\theta_q), \thpub)$\\
        Minimize $\lpriv(\freeze(\theta_q), \thpriv)$\\
        [1.5ex]\hline
    \end{tabular}
    }
    \caption{Our semi-adversarial training scheme.}\label{fig:AdvTraining}
  \end{center}
\end{figure*}

\section{Experimental Results}

\textbf{Accurate main task and poor collateral results.}
In Figures \ref{fig:outputsize_sabo_main} and \ref{fig:outputsize_sabo_coll} we show that the output size has an important influence on the two tasks' performances. For this experiment, we use $\alpha = 1.7$ as detailed in Appendix \ref{app:alpha}, the adversary uses the same CNN as stated above and the main network is a simple feed forward network (FFN) with 4 layers. We observe that both networks behave better when the output size increases, but the improvement is not synchronous which makes it possible to have a main task with high accuracy while the collateral task is still very inaccurate. In our example, this corresponds to an output size between 3 and 5. Note that the collateral result is the accuracy at the distinction task, i.e., the digit is fixed for the adversary which trains to distinguish two fonts during a 50 epoch \textit{recover phase} using 7-fold cross validation, after 50 epochs of \textit{semi-adversarial training} have been spent to reduce leakage from the private network.

\begin{figure}[t]
{\centering
\begin{minipage}{0.49\textwidth}
  \footnotesize
  \centering
\hspace{-15px}
\includegraphics[width=1\textwidth]{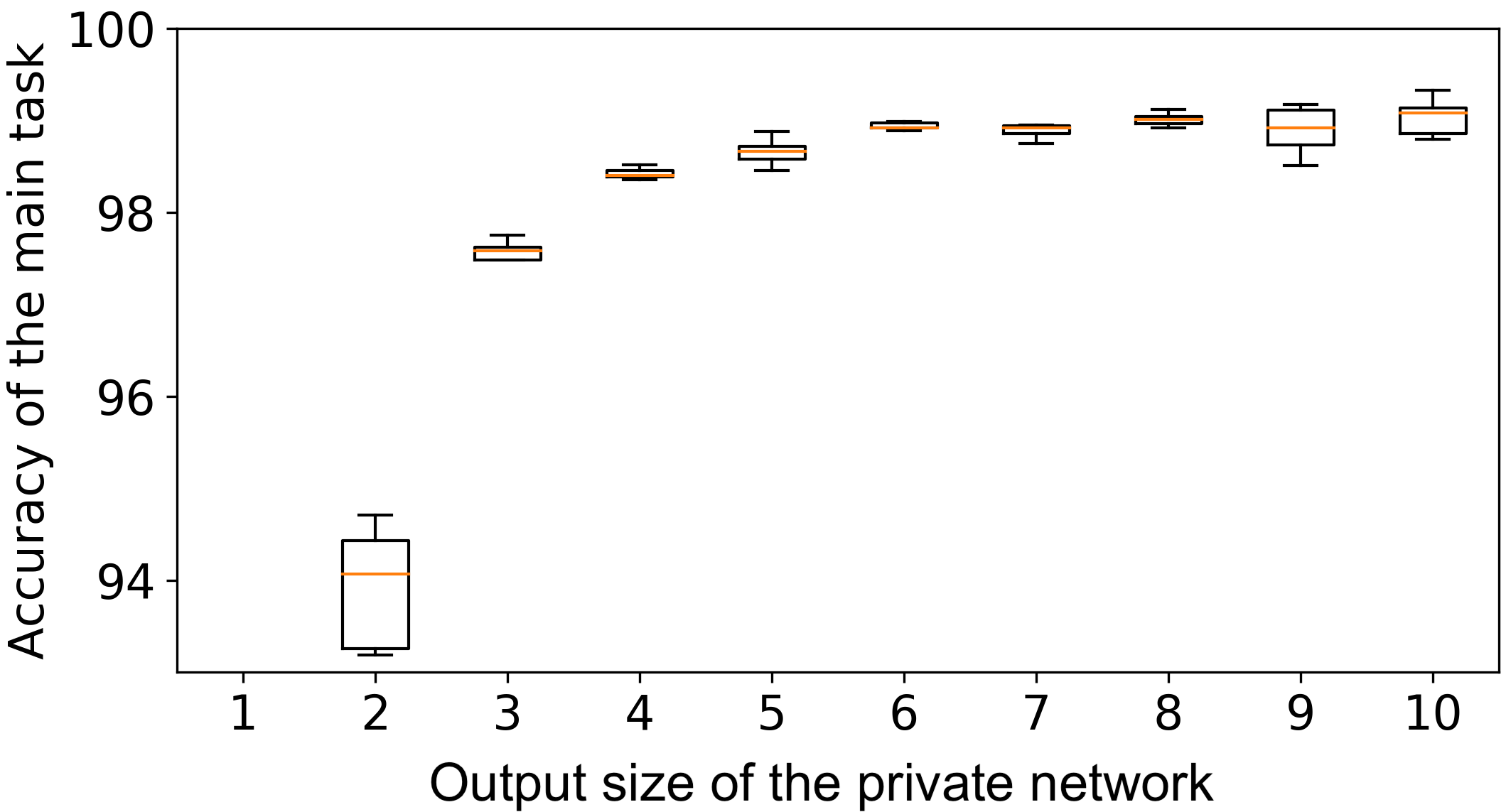}

\vspace*{-.2cm}

  \captionof{figure}{Influence of the output size on the main task accuracy with adversarial training.} 
  \label{fig:outputsize_sabo_main}
\end{minipage}
\hspace{10px}
\begin{minipage}{0.49\textwidth}
\centering
\vspace{-10px}
\includegraphics[width=1\textwidth]{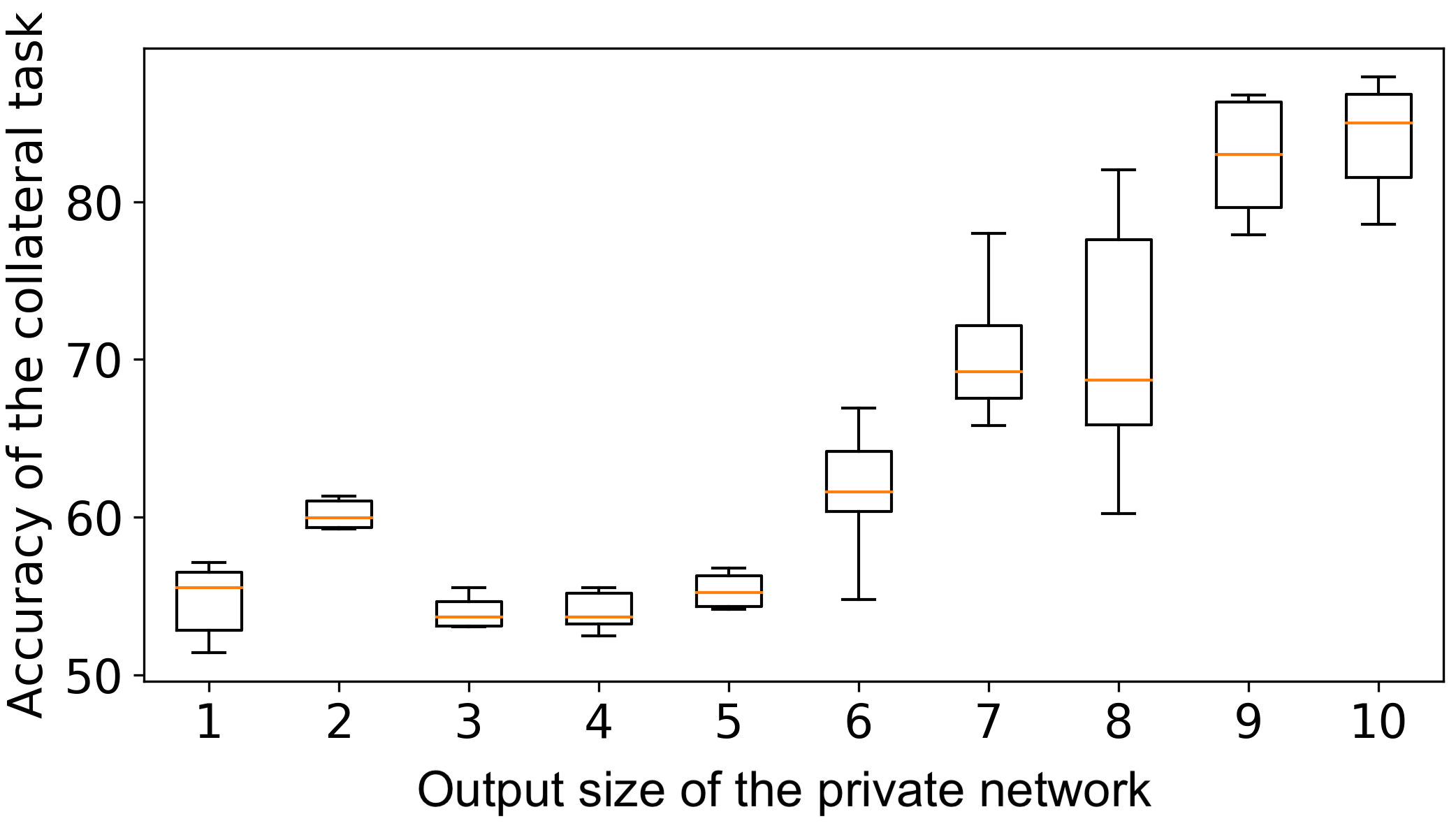}

\vspace*{-.2cm}

\captionof{figure}{Influence of the output size on the collateral task accuracy with adversarial training.}
\label{fig:outputsize_sabo_coll}
\end{minipage}
}
\vspace{-10px}
\end{figure}

\textbf{Generalizing resistance against multiple adversaries.}
In practice, it is very likely that the adversary will use a different model than the one against which the protection has been built. We have therefore investigated how building resistance against a model $M$ can provide resistance against other models. Our empirical results tend to show that models with less parameters than $M$ do not perform well. In return, models with more parameters can behave better, provided that the complexity does not get excessive for the considered task, because it would not provide any additional advantage and would just lead to learning noise. In particular, the CNN already mentioned above seems to be a sufficiently complex model to resist against a wide range of feed forward (FFN) and convolutional networks, as illustrated in Figure \ref{fig:resistanceNN} where the measure used is indistinguishability of the font for a fixed digit. This study is not exhaustive as the adversary can change the activation function (here we use $\mathsf{relu}$) or even the training parameters (optimizer, batch size, dropout, etc.), but these do not seem to provide any decisive advantage.

\begin{figure}[htb!]
{\centering
\begin{minipage}{0.49\textwidth}
  \footnotesize
  \centering
\hspace{-15px}
    \includegraphics[width=1.03\linewidth]{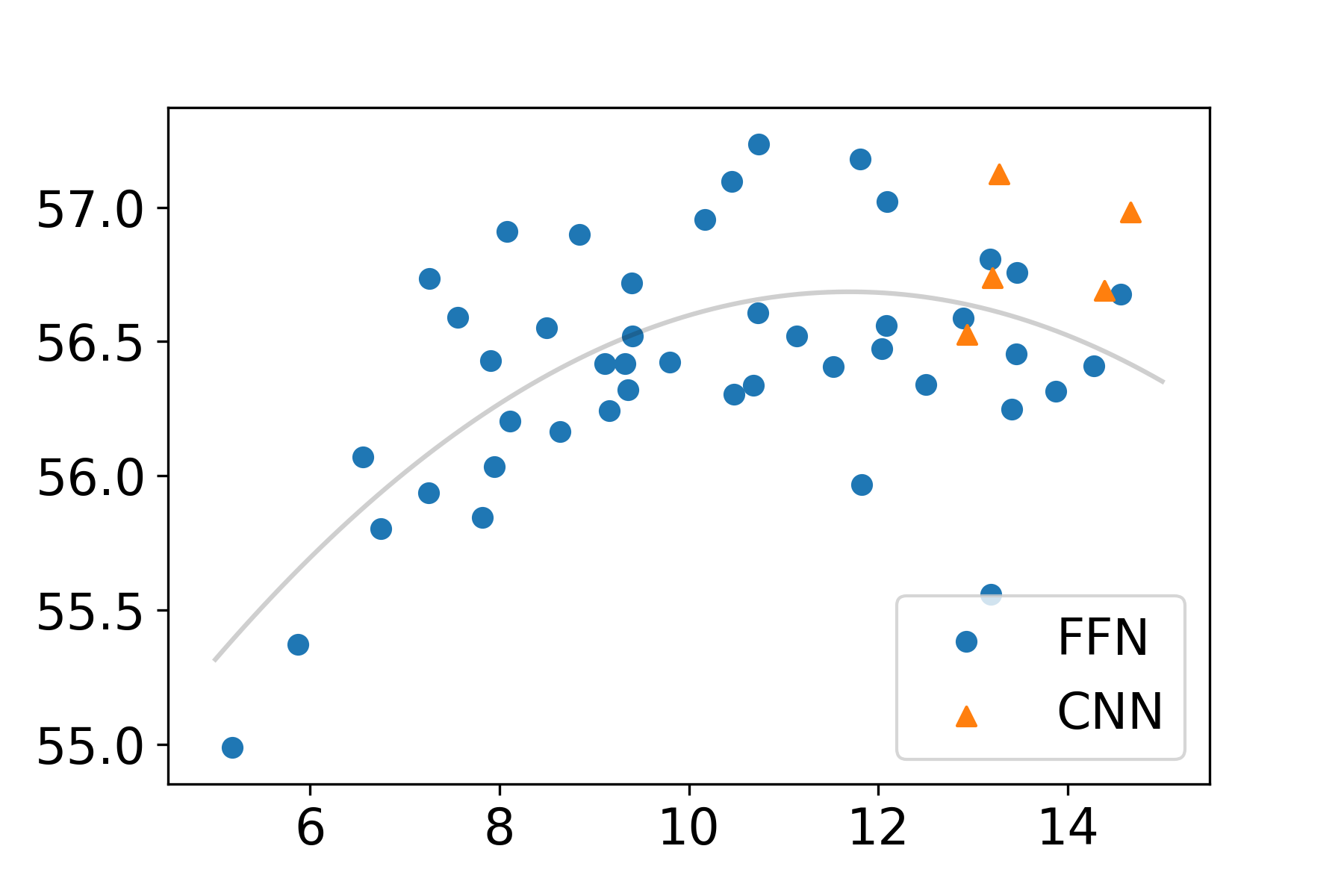}
    
    \vspace*{-.4cm}
  \caption{Collateral accuracy depending of the adversarial network complexity seen as the log of the number of parameters.}\label{fig:resistanceNN}
\end{minipage}
\hspace{10px}
\begin{minipage}{0.49\textwidth}
\centering
\vspace{-10px}
\begin{tabular}{|l|l|}
     \hline
     Linear Ridge Regression & $53.5 \pm 0.5\%$\\
     \hline
     Logistic Regression  & $52.5 \pm 0.6\%$\\
     \hline
     Quad. Discriminant Analysis  & $54.9 \pm 0.3\%$\\
     \hline
     SVM (RBF kernel)  & $57.9 \pm 0.4\%$\\
     \hline
     Gaussian Process Classifier  & $53.8 \pm 0.3\%$\\
     \hline
     Gaussian Naive Bayes  & $53.2 \pm 0.5\%$\\
     \hline
     K-Neighbors Classifier  & $58.1 \pm 0.7\%$\\
     \hline
     Decision Tree Classifier  & $56.8 \pm 0.4\%$\\
     \hline
     Random Forest Classifier  & $58.9 \pm 0.2\%$\\
     \hline
     Gradient Boosting Classifier  & $58.9 \pm 0.2\%$\\
     \hline
    \end{tabular}
    
    \vspace*{-.02cm}
    
  \caption{Accuracy on the distinction task for different adversarial learning models.}\label{tab:resistanceSklearn}
\end{minipage}
}
\vspace{-5px}
\end{figure}

We also assessed the resistance to a large range of other models from the sklearn library \cite{scikit-learn} and report the collateral accuracy in Figure \ref{tab:resistanceSklearn}. As can be observed, some models such as k-nearest neighbors or random forests perform better compared to  neural networks, even if their accuracy remains relatively low. One reason can be that they operate in a very different manner compared to the model on which the adversarial training is performed: k-nearest neighbors for example just considers distances between points.

\textbf{Runtime.}
Training in semi-adversarial mode can take quite a long time depending on the level of privacy one wants to achieve. However, the runtime during the test phase is much faster, it is dominated by the FE scheme part which can be broken down to 4 steps: functional key generation, encryption of the input, evaluation of the function and discrete logarithm. Regarding encryption and evaluation, the main overhead comes from the exponentiations and pairings which are implemented in the crypto library charm \cite{charm13}. In return, the discrete logarithm is very efficient thanks to the reduction of the weights amplitude detailed in Figure \ref{ReduceLeak}.

\begin{table}[htb!]
  \centering
    \begin{tabular}{|l|l|l|l|l|}
     \hline
	 Functional key generation & $94 \pm 5$ms & ~ & Evaluation time & $2.97 \pm 0.07$s\\
     \hline
     Encryption time & $12.1 \pm 0.3$ s & ~ & Discrete logarithms time & $24 \pm 9$ms\\
     \hline
    \end{tabular}
    
    \vspace*{.1cm}
    
  \caption{Average runtime for the FE scheme using a 2,7 GHz Intel Core i7 and 16GB of RAM. }\label{fig:runtime}
\end{table}

Table \ref{fig:runtime} shows that encryption time is longer than evaluation time, but a single encryption can be used with several decryption keys $\dk_{q_i}$ to perform multiple evaluation tasks. 


\section{Conclusion}
We have shown that functional encryption can be used for practical applications where machine learning is used on sensitive data. We have raised awareness about the potential information leakage when not all the network is encrypted and have proposed semi-adversarial training as a solution to prevent targeted sensitive features from leaking for a vast family of adversaries.

However, it remains an open problem to provide privacy-preserving methods for all features except the public ones as they can be hard to identify in advance. On the cryptography side, extension of the range of functions supported in functional encryption would help increase provable data privacy, and adding the ability to hide the function evaluated would be of interest for sensitive neural networks.

\section*{Acknowledgments}

This work was supported in part by the European Community’s Seventh Framework Programme (FP7/2007-2013 Grant Agreement no. 339563 -- CryptoCloud), the European Community’s Horizon 2020 Project FENTEC (Grant Agreement no. 780108), the Google PhD fellowship, and the French FUI ANBLIC Project.


\bibliographystyle{plain}
\bibliography{ref}

\clearpage
\appendix
\section{Functional Encryption and crypto tools}\label{appendix:FEsec}

\subsection{Formal definition of Functional Encryption}\label{appendix:FE}

Functional encryption relies on a pair of keys like in public key encryption: a master secret key $\msk$ and a public key $\pk$. The public key $\pk$ can be shared and is used to encrypt the data, while the master secret key $\msk$   is used to build functional decryption keys $\dk_{f}$ for $f \in \cF$. A user having access to $c$ an encryption of $x$ with $\pk$ and to $\dk_{f}$ can learn $f(x)$ but can't learn anything else about $x$.

We give the definition of Functional Encryption, originally defined in \cite{TCC:BonSahWat11,cryptoeprint:2010:556}.

\begin{definition}[Functional Encryption]
    A {\em functional encryption} scheme $\FE$ for a set of functions $\cF \subseteq \cX \rightarrow \cY$ is a tuple of PPT algorithms
    $\FE = (\SetUp, \allowbreak \KeyGen, \allowbreak  \Enc, \allowbreak  \Dec)$ defined as follows.
    \begin{description}
        \item[$\SetUp(1^{\secpar},\cF)$] takes as input a security parameter $1^{\secpar}$, the set of functions $\cF$, and outputs a master secret key $\msk$ and a public key $\pk$.

        \item[$\KeyGen(\msk,f)$] takes as input the master secret key and a function $f \in \cF$, and outputs a functional decryption key $\dk_{f}$.

        \item[$\Enc(\pk, x)$] takes as input the public key $\pk$ and a
        message $x \in \cX$, and outputs a ciphertext $\ct$.

        \item[$\Dec(\dk_f, \ct)$] takes as input a functional decryption key $\dk_{f}$ and a ciphertext
        $\ct$, and returns an output $y \in \cY \cup \{\bot\}$, where $\bot$ is a special rejection symbol.
    \end{description}
\end{definition}

\subsection{IND-CPA security}\label{app:ind-cpa}
With notations of Appendix \ref{appendix:FE}, for any stateful adversary $\advA$ and any functional encryption scheme $\FE$, we define the following advantage.

\[\Adv^{\FE}_{\advA}(\lambda) := \Pr\left[\beta' = \beta :\begin{array}{l}
(\pk,\msk) \gets \SetUp(1^\lambda,\cF)\\
(x_0,x_1) \gets \advA^{\KeyGen(\msk,\cdot)}(\pk)\\
\beta \getsr \zo, 
\ct \gets \Enc(\pk,x_\beta)\\
\beta' \gets \advA^{\KeyGen(\msk,\cdot)}(\ct)
\end{array}\right] - \frac{1}{2},\]

with the restriction that all queries $f$ that $\advA$ makes to key generation algorithm
$\KeyGen(\msk,\cdot)$ must satisfy $f(x_0)=f(x_1)$.

We say $\FE$ is IND-CPA secure if for all PPT adversaries $\advA$, $\Adv^{\FE}_{\advA}(\lambda) = \negl(\secpar)$\footnote{In cryptography, the security parameter $\lambda$ is a measure of the probability with which an adversary can break the scheme. $\lambda$ or $1^\lambda$ means that the probability of breaking the scheme is $2^{-\lambda}$.}.

\subsection{Bilinear Groups}
\label{sec:group}

Our FE scheme uses bilinear (or \textit{pairing}) groups, whose use in cryptography has been introduced by \cite{BonFra03,JC:Joux04}.
More precisely, given $\lambda$ a security parameter, let $\Gone$ and $\Gtwo$ be two cyclic groups of prime order $p$ (for a $2\lambda$-bit prime $p$) and $g_1$ and $g_2$ their generators, respectively.
The application $e: \Gone \times \Gtwo \rightarrow \Gt$ is a pairing if it is efficiently computable, non-degenerated, and bilinear: $e(g_1^\alpha,g_2^\beta) = e(g_1,g_2)^{\alpha \beta}$ for any $\alpha,\beta\in\Z_p$. Additionally, we define $g_T := e(g_1, g_2)$ which spans the group $\Gt$ of prime order $p$.

We will denote by $\ggen$ a probabilistic polynomial-time (PPT) algorithm that on input $1^\lambda$ returns a description $\PG=(\Gone,\Gtwo,p,g_1,g_2,e)$ of an asymmetric bilinear group. For convenience, given $s = 1$, 2 or $T$, $n \in \N$ and vectors $\vec u := ( u_1 \dots  u_n )\in \Z^n_p$, $\vec v \in \Z_p^{n}$, we denote by $g_s^{\vec u} := (g^{u_1}_s \dots g^{u_n}_s)  \in \mathbb{G}_s^{n}$ and $e(g_1^{\vec u},g_2^{\vec v}) = \prod_{i=1} e(g_1,g_2)^{u_i \cdot v_i} =  e(g_1,g_2)^{\vec u \cdot \vec v} \in \Gt$, where $\vec u \cdot \vec v$ is the inner product, i.e. $\vec u \cdot \vec v := \sum_{i=1}^n u_i v_i$.

\section{Our Quadratic Functional Encryption Scheme}\label{appendix:qfe}

\subsection{Proofs of IND-CPA security and correctness}\label{appendix:QFEproofs}

\subsubsection*{Proof of Security}

To prove security of our scheme, we use the Generic Bilinear Group Model, which captures the fact that no attacks can make use of the representation of group elements. For convenience, we use Maurer's model~\cite{IMA:Maurer05}, where a third party implements the group and gives access to the adversary via handles, providing also equality checking. This is an alternative, but equivalent, formulation of the Generic Group Model, as originally introduced in~\cite{Nechaev94,EC:Shoup97}.

We prove security in two steps: first, we use a master theorem from \cite{C:BCFG17} that relates the security in the Generic Bilinear Group model to a security in a symbolic model. Second, we prove security in the symbolic model. Let us now explain the symbolic model (the next paragraph is taken verbatim from \cite{CCS:ABGW17}).

In the symbolic model, the third party does not implement an actual group, but keeps track of
abstract expressions. For example, consider an experiment where values $x,y$ are sampled from $\Z_p$ and the adversary gets handles to $g^x$ and $g^y$. In the generic model, the third party will choose a group of order $p$, for example $(\Z_p,+)$, will sample values $x,y \leftarrow_{R} \Z_p$ and will give handles to $x$ and $y$. On the other hand, in the symbolic model the sampling  won't be performed and the third party will output handles to $X$ and $Y$, where $X$ and $Y$ are abstract variables. Now, if the adversary asks for equality of the elements associated to the two handles, the answer will be negative in the symbolic model, since abstract variable $X$ is different from abstract variable $Y$, but there is a small chance the equality check succeeds in the generic model (only when the sampling of $x$ and $y$ coincides).

To apply the master theorem, we first need to change the distribution of the security game to ensure that the public key, challenge ciphertext, and functional decryption keys only contain group elements whose exponent is a polynomial evaluated on uniformly random values in $\Z_p$ (this is called polynomially induced distributions in \cite[Definition 10]{C:BCFG17}, and previously in \cite{C:BFFMSS14}). We show that this is possible with only a negligible statistical change in the distribution of the adversary view.

After applying the master theorem from \cite{C:BCFG17}, we prove the security in the symbolic model (cf. Appendix \ref{lem:symb}), which simply consists of checking that an algebraic condition on the scheme in satisfied.

\begin{theorem}[IND-CPA Security in the Generic Bilinear Group Model]\label{thm:FEsec}
For any PPT adversary $\advA$ that performs at most $Q$ group operations against the functional encryption scheme described on \ref{fig:FE}, we have, in the generic bilinear group model:
\[ \Adv^{\FE}_{\advA}(\lambda) \leq \frac{12\cdot(6n+3+Q+Q')^2+1}{p},\] where $Q'$ is the number of queries to $\KeyGen(\msk,\cdot)$.
\end{theorem}
The proof of this result is quite technical and can be found in the dedicated Appendix \ref{app:FEproof}.

\subsubsection*{Proof of Correctness}

For all $i,j\in[n]$, we have:
$$e(g_1^{\vec{a}_i},g_2^{\vec{b}_j}) = g_T^{{\vec{a}_i} \cdot \vec{b}_j} = g_T^{x_i y_j - \gamma s_i t_j}$$
since
\begin{align*}
\vec{a}_i \cdot \vec{b}_j
& = \left( (\matW^{-1})^\top \begin{pmatrix}x_i \\ \gamma s_i\end{pmatrix} \right)^\top
  \cdot \left(\matW \begin{pmatrix}y_j \\ -t_j\end{pmatrix}\right) \\
& = \begin{pmatrix} x_i \\ \gamma s_i \end{pmatrix}^\top \matW^{-1} \matW \begin{pmatrix}y_j \\ -t_j\end{pmatrix}
= {x_i y_j - \gamma s_i t_j}.
\end{align*}
Therefore we have:
\begin{align*}
out &= e(g_1^\gamma,g_2^{q(\vs,\vt)}) \cdot \prod_{i,j} e(g_1^{\vec{a}_i},g_2^{\vec{b}_i})^{q_{i,j}}
 = g_T^{\gamma q(\vs,\vt)} \cdot g_T^{\sum_{i,j} q_{i,j} x_i y_j-\gamma q_{i,j} s_i t_j} \\
& = g_T^{\gamma q(\vs,\vt)} \cdot g_T^{q(\vx,\vy)-\gamma q(\vs,\vt)} = g_T^{q(\vx,\vy)}.
\end{align*}

\subsubsection*{Proof of Complexity}

The complexity can be inferred from the decryption phase as detailed in Figure \ref{fig:FE} and we compare this with previous quadratic FE schemes in Figure  \ref{fig:compare}.

\begin{figure*}[htb!]
    \normalsize
    \centering
    \begin{tabular}{@{}l|c|c|c|c}
        FE scheme & $ct$ &  $\dk_f$ & $\Dec$ & Assumption \\\hline
        \rule{0pt}{1.0\normalbaselineskip}\cite[Sec. 3]{C:BCFG17} & $\Gone^{6n+1} \times \Gtwo^{6n+1}$ &  $\Gone \times \Gtwo$ & $6n^2(E_1+P)+2P$ & SXDH, 3PDDH\\
        \cite[Sec. 4]{C:BCFG17} &  $\Gone^{2n+1} \times \Gtwo^{2n+1}$ &  $\Gone^2$ & $3n^2(E_1+P)+2P$ & GGM \\
        Ours &  $\Gone^{2n+1} \times \Gtwo^{2n}$  &  $\Gtwo$ & $2n^2(E_1+P)+P$  & GGM
    \end{tabular}
    \caption{Performance comparison of FE for quadratic polynomials. $E_1$ and $P$ denote exponentiation in $\Gone$ and pairing evaluation, respectively. Decryption additionally requires solving a discrete logarithm but this computational overhead is the same for all schemes and is therefore omitted here.} \label{fig:compare}
\end{figure*}

\subsection{Detailed equivalence of the FE scheme with a neural network}

\subsubsection*{Proof of Linear Homomorphism}\label{app:QFEequiv}

For all $(\vec x, \vec y) \in \Z_p^n\times\Z_p^n$, and $(\vec u, \vec v) \in \Z_p^n\times\Z_p^n$, given an encryption of $(\vec x, \vec y)$ under the public key $\pk := (g_1^{\vec s}, g_2^{\vec t})$, one can efficiently compute an encryption of $(\vec u^\top \vec x, \vec v^\top \vec y)$ under the public key $\pk' := (g_1^{\vec u^\top \vec s}, g_2^{\vec v^\top \vec t})$. Indeed, given \[\Enc(\pk,(\vec x, \vec y)) := (g_1^\gamma, \{g_1^{\vec a_i}, g_2^{\vec b_i}\}_{i \in [n]}),\]
and $\vec u, \vec v \in \Z_p^n$, one can efficiently compute:  \[(g_1^\gamma, g_1^{\sum_{i \in [n]} u_i \cdot \vec a_i}, g_2^{\sum_{i \in [n]} v_i \cdot \vec b_i}),\]
which is $\Enc(\pk',(\vec u^\top \vec x, \vec v^\top \vec y))$, since:
\begin{align*}
\sum_{i\in [n]} u_i \cdot \vec a_i & = \sum_{i \in [n]} u_i \cdot (\matW^{-1})^\top \begin{pmatrix}x_i \\ \gamma s_i\end{pmatrix}
= (\matW^{-1})^\top \begin{pmatrix} \sum_{i \in [n]} u_i \cdot x_i \\ \gamma \sum_{i \in [n]} u_i \cdot s_i\end{pmatrix} \\
& = (\matW^{-1})^\top \begin{pmatrix} \vec u^\top \vec x \\ \gamma \vec u^\top \vec s\end{pmatrix}.
\end{align*}
Similarly, we have:
\begin{align*}
\sum_{i\in [n]} v_i \cdot \vec b_i = \sum_{i \in [n]} v_i \cdot \matW \begin{pmatrix} y_i \\ -t_i\end{pmatrix} =
  \matW \begin{pmatrix} \vec v^\top \vec y \\ -\vec v^\top \vec t\end{pmatrix}.
\end{align*}

\section{Additional results}

\subsection{Influence of weight compression on the network performance}\label{app:Wcompression}

We show here that we can manage to compress significantly the network weights in order to have a very fast discrete logarithm without modifying the results and conclusions made throughout the article. The main and collateral model follow the same CNN structure as stated above, and the collateral accuracy is reported after 10 epochs of training.

\begin{table}[htb!]
  \centering
    \begin{tabular}{|l|l|}
     \hline
     Main accuracy  with compression & $97.72 \pm  0.30 ~ \%$\\
     \hline
     Collateral accuracy with compression & $55.27 \pm  0.41 ~ \%$\\
     \hline
    \end{tabular}
    \vspace{10px}
  \caption{Impact of weight compression on the main and collateral accuracies}\label{tab:Wcompression}
\end{table}

\subsection{Influence of alpha during adversarial training}\label{app:alpha}

To choose the best value for $\alpha$, we have chosen an output size of 4 which allows us to keep a very high main accuracy while reducing significantly the collateral one, as shown in Figure \ref{fig:pareto_nosabo}. We observe that the semi-adversarial training does not affect much the main accuracy for a large range of values for $\alpha$, while its impact on the collateral accuracy is decisive. Figure \ref{fig:alpha} illustrates the role of $\alpha$ and justify our choice of $\alpha = 1.7$. For this experiment, we have chosen for both networks a simple feed forward with a hidden layer of 32 neurons.

\begin{figure}[htb!]
  \centering
  \includegraphics[width=0.70\linewidth]{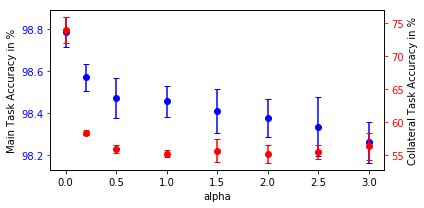}
  \caption{Trade-off between the main and collateral tasks accuracies as a function of $\alpha$}\label{fig:alpha}
\end{figure}

\section{Security proof of our FE scheme}\label{app:FEproof}

\begin{proof} ~For any experiment $\Exp$, adversary $\advA$, and security parameter $\lambda \in \N$, we use the notation: $\adv_{\Exp}(\advA) := \Pr[1 \gets \Exp(1^\lambda,\advA)]$, where the probability is taken over the random coins of $\Exp$ and $\advA$.

\begin{figure*}
  \begin{center}
    \begin{tabular}{|l|l|}\hline
        \underline{$\Exp_1(1^\lambda,\advA)$:} & \underline{$\KeyGen(\msk,f)$:}\\
        $(\Gone,\Gtwo,p,g_1,g_2,e) \gets \ggen(1^\lambda)$, $\vs,\vt \getsr \Z^n_p$ & return $(g_2^{f(\vs,\vt)},f)$.\\
        $a,b,c,d \getsr \Z_p$, set $\PG:= (\Gone,\Gtwo,p,g_1^{ad-bc},g_2,e)$ & \\
        $\msk := (\vs,\vt)$, $\pk := \left(\PG,g_1^{(ad-bc)\vs}, g_2^{\vt}\right)$ & \\
        $\Big((\vec x^{(0)},\vec y^{(0)}), (\vec x^{(1)}, \vec y^{(1)})\Big) \gets \advA^{\KeyGen(\msk,\cdot)}(\pk)$ &\\
        $\beta \getsr \zo$, $\gamma \getsr \Z_p$ &\\
        for all $i \in [n]$, $\vec{a}_i := \begin{pmatrix}d & -c \\ -b & a \end{pmatrix} \begin{pmatrix}x^{(\beta)}_i \\ \gamma s_i\end{pmatrix}$, $\vec{b}_i := \begin{pmatrix}a & b \\ c & d\end{pmatrix} \begin{pmatrix} y^{(\beta)}_i \\ -t_j\end{pmatrix}$ &\\
        $ct =: \big(g_1^{\gamma(ad-bc)}, \{g_1^{\vec{a}_i},g_2^{\vec{b}_i}\}_{i \in [n]}\big)$ &\\
        $\beta' \gets \advA^{\KeyGen(\msk,\cdot)}(\pk,ct)$ &\\
        Return $1$ if $\beta' = \beta$ and for all queried $f$, $f(\vec x^{(0)},\vec y^{(0)}) = f(\vec x^{(1)}, \vec y^{(1)})$.&\\\hline
        \end{tabular}
    \caption{Experiment $\Exp_1$, for the proof of Theorem \ref{thm:FEsec}.}\label{fig:G1}
  \end{center}
\end{figure*}

While we want to prove the security result in the real experiment $\Exp_0$, in which the adversary has to guess $\beta$, we slightly modify it into the hybrid experiment $\Exp_1$, described in \ref{fig:G1}:
we write the matrix $\matW \getsr \GL_2$ used in the challenge ciphertext as $\matW := \begin{pmatrix} a & b \\ c & d \end{pmatrix}$, chosen from the beginning. Then $\matW^{-1} = \frac{1}{ad-bc} \begin{pmatrix} d & -b \\ -c & a \end{pmatrix}$.

The only difference with the IND-CPA security game as defined in Appendix \ref{app:ind-cpa}, is that we change the generator $g_1 \getsr \Gone^*$ into $g_1^{ad-bc}$ for $a,b,c,d \getsr \Z_p$, which only changes the distribution of the game by a statistical distance of at most $\frac{3}{p}$ (this is obtained by computing the probability that $ad-bc=0$ when $a,b,c,d \getsr \Z_p$).
Thus, \[\Adv^{\FE}_{\advA}(\lambda) = \adv_0(\advA) \leq \adv_1(\advA) + \frac{3}{p}.\]

Note that in $\Exp_1$, the public key, the challenge ciphertext and the functional decryption keys only contain group elements whose exponents are polynomials evaluated on random inputs (as opposed to $g_1^{\matW^{-1}}$, for instance). This is going to be helpful for the next step of the proof, which uses the generic bilinear group model.

Next, we make the generic bilinear group model assumption, which intuitively says that no PPT adversary can exploit the structure of the bilinear group to perform better attacks than generic adversaries. That is, we have (with $\Exp_2$  defined in \ref{fig:GGM}):
\[\max_{\mbox{\begin{scriptsize}PPT \end{scriptsize}} \advA}\Big(\adv_1(\advA)\Big) = \max_{\mbox{\begin{scriptsize}PPT \end{scriptsize}}\advA}\Big(\adv_2(\advA)\Big).\]

\begin{figure*}
    \begin{center}\fbox{
            \begin{minipage}[t]{0.95\textwidth}
                \underline{$\Exp_2(1^\lambda,\advA)$:}\\[0.2ex]
                $L_1=L_2= L_T:= \emptyset$, $Q_\sk := \emptyset$, $\vs,\vt \getsr \Z_p^n$, $a,b,c,d \getsr \Z_p$, $\append(L_1,(ad-bc) \cdot \vec s)$, $\append(L_2,\vec t)$, $\beta \getsr \zo$\\
                $\Big((\vec x^{(0)},\vec y^{(0)}),(\vec x^{(1)},\vec y^{(1)}) \Big) \gets \advA^{\cO_{\add},\cO_{\pair},\cO_{\sk},\cO_{\eq}}(1^\lambda,p)$\\
                $\cO_{\chal}\Big((\vec x^{(0)},\vec y^{(0)}),(\vec x^{(1)},\vec y^{(1)}) \Big)$\\
                $\beta'\gets \advA^{\cO_{\add},\cO_{\pair},\cO_{\sk},\cO_{\eq}}(1^\lambda,p)$\\
                If $\beta=\beta'$, and for all $f \in Q_\sk$, $f(\vec x^{(0)},\vec y^{(0)})=f(\vec x^{(1)},\vec y^{(1)})$, output $1$. Otherwise, output $0$.\\
                \\
                \underline{$\cO_{\add}(s\in \{1,2,T\},i,j \in \N)$:}\\[0.2ex]
                $\append(L_s,L_s[i]+ L_s[j])$.
                \\
                \\
                \underline{$\cO_{\pair}(i,j\in \N)$:}\\[0.2ex]
                $\append(L_T,L_1[i]\cdot L_2[j])$.\\
                \\
                \underline{$\cO_{\chal}\Big((\vec x^{(0)},\vec y^{(0)}),(\vec x^{(1)},\vec y^{(1)}) \Big)$:}\\[0.2ex]
                $\gamma \getsr \Z_p$, $\append(L_1,\gamma(ad-bc))$\\
                for all $i \in [n]$, $\vec{a}_i := \begin{pmatrix} d & -c \\ -b & a\end{pmatrix} \begin{pmatrix} x^{(\beta)}_i \\ \gamma s_i \end{pmatrix}$, $\append(L_1,\vec{a}_i)$,
                $\vec{b}_i := \begin{pmatrix} a & b \\ c & d \end{pmatrix} \begin{pmatrix} y^{(\beta)}_i \\ - t_i \end{pmatrix}$, $\append(L_2,\vec{b}_i)$.\\
                \\
                \underline{$\cO_{\sk}(f \in \cF_{n,B_x,B_y,B_f})$:}\\[0.2ex]
                $\append(L_2,f(\vec s, \vec t))$, $Q_\sk:= Q_\sk \cup \{f\}$.\\
                \\
                \underline{$\cO_{\eq}(s \in \{1,2,T\},i,j \in \N)$:}\\[0.2ex]
                Output $1$ if $L_s[i] = L_s[j]$, $0$ otherwise
        \end{minipage}}
    \end{center}
    \caption{Experiment $\Exp_2$. Wlog. we assume no query contains indices $i,j\in \N$ that exceed the size of the involved lists.}\label{fig:GGM}
\end{figure*}

In this experiment, we denote by $\emptyset$ the empty list, by $\append(L,x)$ the addition of an element $x$ to the list $L$, and for any $i \in \N$, we denote by $L[i]$ the $i$'th element of the list $L$ if it exists (lists are indexed from index $1$ on), or $\bot$ otherwise.

Thus, it suffices to show that for any PPT adversary $\advA$, $\adv_2(\advA)$ is negligible in $\lambda$.
The experiment $\Exp_2$ defined in Figure \ref{fig:GGM} falls into the general class of simple interactive decisional problems from \cite[Definition 14]{C:BCFG17}. Thus, we can use their master theorem \cite[Theorem 7]{C:BCFG17}, which, for our particular case (setting the public key size $N:=2n+2$, the key size $c=1$, the ciphertext size $c^*:=4n+1$, and degree $d=6$ in \cite[Theorem 7]{C:BCFG17}) states that:
     \[ \adv_2(\advA) \leq \frac{12\cdot(6n+3+Q+Q')^2}{p},\] where $Q'$ is the number of queries to $\cO_{\sk}$, and $Q$ is the number of group operations, that is, the number of calls to oracles $\cO_{\add}$ and $\cO_{\pair}$, provided the following algebraic condition is satisfied:
     \begin{align*}& \{\matM \in \Z_p^{(3n+2) \times (3n+Q'+1)}: \mathsf{Eq}_0(\matM)\} \\
       & = \{\matM \in \Z_p^{(3n+2) \times (3n+Q'+1)}: \mathsf{Eq}_1(\matM) \},\end{align*}
     where for all $\matM$, $b \in \zo$,
     \[\mathsf{Eq}_b(\matM): \begin{pmatrix} 1 \\ (AD-BC) \vec S \\ (AD-BC)\Gamma \\ D \vec x^{(b)} - \Gamma C \vec S \\ -B \vec x^{(b)} + \Gamma A \vec S\end{pmatrix}^\top \matM \begin{pmatrix} 1 \\ \vec T \\ A \vec y^{(b)} - B \vec T \\ C \vec y^{(b)} - D \vec T \\ (f(\vec S,\vec T))_{ f \in  Q_\sk}\end{pmatrix} = 0,\]
     where the equality is taken in the ring $\Z_p[\vec S, \vec T,A,B,C,D,\Gamma]$, and $0$ denotes the zero polynomial. Intuitively, this condition captures the security at a symbolic level: it holds for schemes that are not trivially broken.
     The latter means that computing a linear combination in the exponents of target group elements that can be obtained from $\pk$, the challenge ciphertext, and functional decryption keys, does not break the security of the scheme.
     We prove this condition is satisfied in \ref{lem:symb} below.
\end{proof}

\begin{lemma}[Symbolic Security]\label{lem:symb}
    For any $(\vec x^{(0)},\vec y^{(0)}), (\vec x^{(1)},\vec y^{(1)}) \in Z_p^{2n}$, and any set $Q_\sk \subseteq \cF_{n,B_x,B_y,B_f}$ such that for all $f \in Q_\sk$, $f(\vec x^{(0)}, \vec y^{(0)}) = f(\vec x^{(1)}, \vec y^{(1)})$, we have:
    \begin{align*}& \{\matM \in \Z_p^{(3n+2) \times (3n+Q'+1)}: \mathsf{Eq}_0(\matM)\} \\ & = \{\matM \in \Z_p^{(3n+2) \times (3n+Q'+1)}: \mathsf{Eq}_1(\matM) \},\end{align*}
    where for all $\matM$, $b \in \zo$,
    \[\mathsf{Eq}_b(\matM): \begin{pmatrix} 1 \\ (AD-BC) \vec S \\ (AD-BC)\Gamma \\ D \vec x^{(b)} - \Gamma C \vec S \\ -B \vec x^{(b)} + \Gamma A \vec S\end{pmatrix}^\top \matM \begin{pmatrix} 1 \\ \vec T \\ A \vec y^{(b)} - B \vec T \\ C \vec y^{(b)} - D \vec T \\ (f(\vec S, \vec T))_{ f \in  Q_\sk}\end{pmatrix} = 0,\]
    where the equality is taken in the ring $\Z_p[\vec S, \vec T,A,B,C,D,\Gamma]$, and $0$ denotes the zero polynomial.
\end{lemma}

\begin{proof}Let $b \in \zo$, and $\matM \in \Z_p^{(3n+2) \times (3n+Q'+1)}$ that satisfies $\mathsf{Eq}_b(\matM)$. We prove it also satisfies $\mathsf{Eq}_{1-b}(\matM)$. To do so, we use the following rules:
    \begin{description}
        \item[Rule 1]: for all $P,Q,R \in \Z_p[\vec S, \vec T,A,B,C,D,\Gamma]$, with $\deg(P) \geq 1$, if $P \cdot Q+R = 0$ and $R$ is not a multiple of $P$, then $Q=0$ and $R=0$.
        \item[Rule 2]: for all $P \in \Z_p[\vec S, \vec T,A,B,C,D,\Gamma]$, any variable $X$ among the set $\{\vec S, \vec T,A,B,C,D,\Gamma\}$, and any $x \in \Z_p$, $P=0$ implies $P(X:=x)=0$, where $P(X:=x)$ denotes the polynomial $P$ evaluated on $X=x$.
    \end{description}
    Evaluating $\mathsf{Eq}_b(\matM)$ on $B=D=0$ (using \textbf{Rule 2}), then using \textbf{Rule 1} on $P=C \Gamma S_i T_j$ for all $i,j \in [n]$, we obtain that:
    \[\matM_{n+2+i} \begin{pmatrix} 0 \\ \vec T \\ \zero \\ \zero \\ (f(\vec S, \vec T))_{ f \in  Q_\sk}\end{pmatrix} = 0,\]
    where $\matM_{n+2+i}$ denotes the $n+2+i$'th row of $\matM$.

    Similarly, using \textbf{Rule 1} on $P=\Gamma A S_i T_j$ for all $i,j \in [n]$, we obtain that:
    \[\matM_{2n+2+i} \begin{pmatrix} 0 \\ \vec T \\ \zero \\ \zero \\ (f(\vec S, \vec T))_{ f \in  Q_\sk}\end{pmatrix} = 0.\]

    Thus, we have:
    \begin{equation}\label{eqn1}
        \forall \beta \in \zo: \begin{pmatrix} 0 \\ \zero \\ 0 \\ D \vec x^{(\beta)} - \Gamma C \vec S \\ -B \vec x^{(\beta)} + \Gamma A \vec S\end{pmatrix}^\top \matM \begin{pmatrix} 0 \\ \vec T \\ \zero \\ \zero \\ (f(\vec S, \vec T))_{ f \in  Q_\sk}\end{pmatrix} = 0.
        \end{equation}

    Using \textbf{Rule 1} on $P=(AD-BC) S_i B T_j$ for all $i,j \in [n]$ in the equation $\mathsf{Eq}_b(\matM)$, we get that the coefficient $M_{i+1,n+1+j}=0$ for all $i ,j \in [n]$. Similarly, using \textbf{Rule 1} on $P=(AD-BC) S_i D T_j$ for all $i,j \in [n]$, we get $M_{i+1,2n+1+j}=0$ for all $i ,j \in [n]$. Then, using \textbf{Rule 1} on $P=(AD-BC) \Gamma B T_j$ for all $j \in [n]$, we get $M_{n+2,n+1+j}=0$ for all $j \in [n]$. Finally, using \textbf{Rule 1} on $P=(AD-BC) \Gamma D T_j$ for all $j \in [n]$, we get $M_{n+2,2n+1+j}=0$ for all $j \in [n]$. Overall, we obtain:
    \begin{equation}\label{eqn2}
    \forall \beta \in \zo: \begin{pmatrix} 0 \\ (AD-BC) \vec S \\ (AD-BC)\Gamma \\ \zero \\ \zero \end{pmatrix}^\top \matM \begin{pmatrix} 0 \\ \zero \\ A \vec y^{(\beta)} - B \vec T \\ C \vec y^{(\beta)} - D \vec T \\ \zero \end{pmatrix} = 0.
    \end{equation}

    We write:
    \begin{align*}
        &\begin{pmatrix} 0 \\ \zero \\ 0 \\ D \vec x^{(b)} - \Gamma C \vec S \\ -B \vec x^{(b)} + \Gamma A \vec S\end{pmatrix}^\top \matM \begin{pmatrix} 0 \\ \zero \\ A \vec y^{(b)} - B \vec T \\ C \vec y^{(b)} - D \vec T \\ \zero \end{pmatrix} \\
        & = \sum_{i,j \in [n]} \begin{pmatrix} D x_i^{(b)} - \Gamma C S_i \\ -B x_i^{(b)} + \Gamma A S_i\end{pmatrix}^\top \\
        & \times \left( m_{i,j}^{(1)} \begin{pmatrix} 1 & 0 \\ 0 & 1 \end{pmatrix} + m_{i,j}^{(2)} \begin{pmatrix} 1 & 0 \\ 0 & 0 \end{pmatrix} + m_{i,j}^{(3)} \begin{pmatrix} 0 & 0 \\ 1 & 0 \end{pmatrix} + m_{i,j}^{(4)} \begin{pmatrix} 0 & 1 \\ 0 & 0 \end{pmatrix}\right) \\
        & \times \begin{pmatrix} A y_j^{(b)} - B T_j \\ C y_j^{(b)} - D T_j\end{pmatrix}
        \end{align*}

    Evaluating the equation $\mathsf{Eq}_b(\matM)$ on $C=D=0$ (by \textbf{Rule 2}), then using \textbf{Rule 1} on $P=\Gamma A B S_i T_j$ for all $i,j \in [n]$, we obtain $m_{i,j}^{(3)}=0$ for all $i,j \in [n]$.
    Evaluating the equation $\mathsf{Eq}_b(\matM)$ on $A=B=0$ (by \textbf{Rule 2}), then using \textbf{Rule 1} on $P=\Gamma C D S_i T_j$ for all $i,j \in [n]$, we obtain $m_{i,j}^{(4)}=0$ for all $i,j \in [n]$.
    Evaluating the equation $\mathsf{Eq}_b(\matM)$ on $A = B = C = D =1$ (using \textbf{Rule 2}), then using \textbf{Rule 1} on $P=\Gamma S_i T_j$ for all $i,j \in [n]$, using the fact that $m_{i,j}^{(3)}=m_{i,j}^{(4)}=0$ and \eqref{eqn1}, we obtain $m_{i,j}^{(2)}=0$ for all $i,j \in [n]$.
    Using \textbf{Rule 1} on $P=\Gamma (AD-BC)S_i T_j$ for all $i,j \in [n]$ in the equation $\mathsf{Eq}_b(\matM)$, we obtain that for all $i,j \in [n]$,
    $$m_{i,j}^{(1)} = \matM_{n+2} \begin{pmatrix} 0 \\ \zero \\ \zero \\ \zero \\ (f_{i,j})_{f \in Q_\sk}\end{pmatrix},$$
    where $\matM_{n+2}$ is the $n+2$'th row of $\matM$.
    
    Putting everything together, can write
    \begin{align*}
      \begin{pmatrix} 0 \\ \zero \\ 0 \\ D \vec x^{(b)} - \Gamma C \vec S \\ -B \vec x^{(b)} + \Gamma A \vec S\end{pmatrix}^\top \matM \begin{pmatrix} 0 \\ \zero \\ A \vec y^{(b)} - B \vec T \\ C \vec y^{(b)} - D \vec T \\ \zero \end{pmatrix}
    \end{align*}
    as
    \begin{align*}
       & (AD-BC) \matM_{n+2} \begin{pmatrix} 0 \\ \zero \\ \zero \\ \zero \\ \big(f(\vec x^{(b)},\vec y^{(b)}) - \Gamma f(\vec s,\vec t)\big)_{f \in Q_\sk}\end{pmatrix} \\
       & = (AD-BC) \matM_{n+2} \begin{pmatrix} 0 \\ \zero \\ \zero \\ \zero \\ \big(f(\vec x^{(1-b)},\vec y^{(1-b)}) - \Gamma f(\vec s,\vec t)\big)_{f \in Q_\sk}\end{pmatrix} \\
       & = \begin{pmatrix} 0 \\ \zero \\ 0 \\ D \vec x^{(1-b)} - \Gamma C \vec S \\ -B \vec x^{(1-b)} + \Gamma A \vec S\end{pmatrix}^\top \matM \begin{pmatrix} 0 \\ \zero \\ A \vec y^{(b)} - B \vec T \\ C \vec y^{(b)} - D \vec T \\ \zero \end{pmatrix} \numberthis \label{eqn3}
    \end{align*}
    where we use the fact that for all $f \in Q_\sk$, we have the equality
    $f(\vec x^{(b)},\vec y^{(b)})=f(\vec x^{(1-b)},\vec y^{(1-b)})$.

    Evaluating equation $\mathsf{Eq}_b(\matM)$ on $A=B=D=0$ (by \textbf{Rule 2}), then using \textbf{Rule 1} on $\Gamma S_i C$ for all $i \in [n]$, and using \eqref{eqn1} and \eqref{eqn3}, we obtain that the coefficient $M_{n+2+i,1}=0$ for all $i \in [n]$.
    Evaluating $\mathsf{Eq}_b(\matM)$ on $B=C=D=0$ (by \textbf{Rule 2}), then using \textbf{Rule 1} on $\Gamma S_i A$ for all $i \in [n]$, and using \eqref{eqn1} and \eqref{eqn3}, we obtain that the coefficient $M_{2n+2+i,1}=0$ for all $i \in [n]$. Thus, we have:
    \begin{equation}\label{eqn4}
    \forall \beta \in \zo: \begin{pmatrix} 0 \\ \zero \\ 0 \\ D \vec x^{(\beta)} - \Gamma C \vec S \\ -B \vec x^{(\beta)} + \Gamma A \vec S\end{pmatrix}^\top \matM \begin{pmatrix} 1 \\ \zero \\ \zero \\ \zero \\ \zero \end{pmatrix} = 0.
    \end{equation}

    Evaluating equation $\mathsf{Eq}_b(\matM)$ on $A=C=D=0$ (by \textbf{Rule 2}), then using \textbf{Rule 1} on $B T_j$ for all $i \in [n]$, and using \eqref{eqn3}, we obtain that the coefficient $M_{1,n+1+j}=0$ for all $j \in [n]$.
    Evaluating $\mathsf{Eq}_b(\matM)$ on $A=B=C=0$ (by \textbf{Rule 2}), then using \textbf{Rule 1} on $D T_j$ for all $j \in [n]$, and using \eqref{eqn3}, we obtain that the coefficient $M_{1,2n+1+j}=0$ for all $j \in [n]$. Thus, we have:
    \begin{equation}\label{eqn5}
    \forall \beta \in \zo: \begin{pmatrix} 1 \\ \zero \\ 0 \\ \zero \\ \zero \end{pmatrix}^\top \matM \begin{pmatrix} 0 \\ \zero \\ A \vec y^{(\beta)} - B \vec T \\ C \vec y^{(\beta)} - D \vec T \\ \zero \end{pmatrix} = 0.
    \end{equation}
    Overall, we have:
        \begin{align*}
        \mathsf{Eq}_b(\matM): &\begin{pmatrix} 1 \\ (AD-BC) \vec S \\ (AD-BC)\Gamma \\ D \vec x^{(b)} - \Gamma C \vec S \\ -B \vec x^{(b)} + \Gamma A \vec S\end{pmatrix}^\top \matM \begin{pmatrix} 1 \\ \vec T \\ A \vec y^{(b)} - B \vec T \\ C \vec y^{(b)} - D \vec T \\ (f(\vec S, \vec T))_{ f \in  Q_\sk}\end{pmatrix} = 0
        \end{align*}
        which implies the following relation, under~\eqref{eqn1},\eqref{eqn2},\eqref{eqn4},\eqref{eqn5}
        \begin{align*}
&\begin{pmatrix} 1 \\ (AD-BC) \vec S \\ (AD-BC)\Gamma \\ \zero \\  \zero \end{pmatrix}^\top \matM \begin{pmatrix} 1 \\ \vec T \\ \zero \\ \zero \\ (f(\vec S, \vec T))_{ f \in  Q_\sk}\end{pmatrix} \\
        &+ \begin{pmatrix} 0 \\ \zero \\ 0 \\ D \vec x^{(b)} - \Gamma C \vec S \\ -B \vec x^{(b)} + \Gamma A \vec S\end{pmatrix}^\top \matM \begin{pmatrix} 0 \\ \zero \\ A \vec y^{(b)} - B \vec T \\ C \vec y^{(b)} - D \vec T \\ \zero\end{pmatrix} = 0 \\
        \end{align*}
        and then, under~\eqref{eqn3}
        \begin{align*}
        &\begin{pmatrix} 1 \\ (AD-BC) \vec S \\ (AD-BC)\Gamma \\ \zero \\  \zero \end{pmatrix}^\top \matM \begin{pmatrix} 1 \\ \vec T \\ \zero \\ \zero \\ (f(\vec S, \vec T))_{ f \in  Q_\sk}\end{pmatrix} \\
        & + \begin{pmatrix} 0 \\ \zero \\ 0 \\ D \vec x^{(1-b)} - \Gamma C \vec S \\ -B \vec x^{(1-b)} + \Gamma A \vec S\end{pmatrix}^\top \matM \begin{pmatrix} 0 \\ \zero \\ A \vec y^{(1-b)} - B \vec T \\ C \vec y^{(1-b)} - D \vec T \\ \zero\end{pmatrix} = 0.
        \end{align*}
        Under~\eqref{eqn1},\eqref{eqn2},\eqref{eqn4},\eqref{eqn5}, this implies
        \begin{align*}
        \mathsf{Eq}_{1-b}(\matM): &\begin{pmatrix} 1 \\ (AD-BC) \vec S \\ (AD-BC)\Gamma \\ D \vec x^{(1-b)} - \Gamma C \vec S \\ -B \vec x^{(1-b)} + \Gamma A \vec S\end{pmatrix}^\top \matM \begin{pmatrix} 1 \\ \vec T \\ A \vec y^{(1-b)} - B \vec T \\ C \vec y^{(1-b)} - D \vec T \\ (f(\vec S, \vec T))_{ f \in  Q_\sk}\end{pmatrix} = 0
        \end{align*}
    
\end{proof}

\end{document}